\documentclass[11pt]{article}
\usepackage[utf8]{inputenc}

\usepackage{amsmath,algorithmic}
\usepackage{amssymb,bbm,amsthm}
\usepackage{fullpage,xcolor}
\usepackage{graphicx}
\usepackage{dsfont}
\usepackage{algorithm}
\usepackage{setspace}
\usepackage{appendix}
\usepackage{comment}
\usepackage{multirow}
\usepackage{hyperref}
\usepackage{authblk}
\usepackage{xr}

\DeclareMathOperator*{\argmax}{argmax}

\makeatletter
\newcommand{\vast}{\bBigg@{4}}
\newcommand{\Vast}{\bBigg@{5}}
\makeatother

\allowdisplaybreaks

\usepackage{natbib}
 \bibpunct[, ]{(}{)}{,}{a}{}{,}%

\newtheorem{theorem}{Theorem}
\newtheorem{lemma}{Lemma}
\newtheorem{corollary}{Corollary}

\begin{document}
\title{Filtered Poisson Process Bandit on a Continuum}
\author[1]{James A. Grant\thanks{j.grant@lancaster.ac.uk; corresponding author}}
\author[2]{Roberto Szechtman\thanks{rszechtm@nps.edu}}
\affil[1]{Department of Mathematics and Statistics, Lancaster University, UK}
\affil[2]{Department of Operations Research, Naval Postgraduate School, CA, USA}
\maketitle

\onehalfspacing
\maketitle

\abstract{
We consider a version of the continuum armed bandit where an action induces a filtered realisation of a non-homogeneous Poisson process. Point data in the filtered sample are then revealed to the decision-maker, whose reward is the total number of revealed points. Using knowledge of the function governing the filtering, but without knowledge of the Poisson intensity function, the decision-maker seeks to maximise the expected number of revealed points over $T$ rounds. We propose an upper confidence bound algorithm for this problem utilising data-adaptive discretisation of the action space. This approach enjoys $\tilde{O}(T^{2/3})$ regret under a Lipschitz assumption on the reward function. We provide lower bounds on the regret of any algorithm for the problem, via new lower bounds for related finite-armed bandits, and show that the orders of the upper and lower bounds match up to a logarithmic factor.  
}


\section{Introduction}

The challenge of detecting interesting events, using limited resources, arises in numerous settings. In a defence context, surveillance teams wish to observe suspicious activity or gain intelligence. In ecological and environmental data collection, scientists wish to observe behaviours of endangered species or record notable measurements of environmental variables. In manufacturing and logistics settings, it is desirable to observe faults in machine operation or a supply chain. 

However, in all of these settings, practitioners may face the problem of having insufficient resource to observe everything they wish to, and must optimise their resource allocation to maximise the detection of events. In these settings ``resource'' may refer to human searchers, fixed or mobile sensors, cameras, or a variety of other equipment with a capacity to observe events of interest.

Two factors play a particularly important role in the rate of detection. Crudely put, these are where we look, and how good we are at looking. In any of these settings, we can only expect to observe events in locations (spatial or temporal) where we deploy resource. Further, the precision of the detection may also be affected by how resource is deployed. If resource is spread over a large region, the probability of detecting events within this region may be lower than if focused on a small area.

Inspired by these challenges, we consider a stylised model of resource allocation which captures the challenge of balancing coverage and detection probability. This framework is sufficiently abstract to model problems across the various aforementioned applications and beyond.

Consider a decision-maker who aims to detect the maximum number of events occurring according to a Non-homogeneous Poisson process (NHPP) on a segment $[0,1]$. The decision-maker selects a point $y \in [0,1]$ and then sweeps the sub-segment $[0,y]$ searching for events. However, the decision-maker's search is imperfect, in that events in $[0,y]$ are detected, independently of each other, with \emph{filtering probability} $\gamma(y)$, where $\gamma: [0,1]\rightarrow [0,1]$, is a known, nonincreasing function. The expected number of events detected by the decision-maker on a single sweep is then determined by the filtering probability, and the cumulative intensity function (CIF) of the NHPP, \begin{displaymath}
\Lambda(y)=\int_0^y \lambda(z)dz, \enspace \forall ~y \in [0,1]
\end{displaymath} where $\lambda:[0,1]\rightarrow \mathbb{R}$ is the rate function of the NHPP. Given the decision-maker chooses to sweep $[0,y]$, the number of events detected has a Poisson$(\Lambda(y)\gamma(y))$ distribution.

Figure \ref{fig::illustration} illustrates this process. An example intensity function $\lambda$ is represented by the blue curve and a function $\gamma$ giving the filtering probability is given by the black curve. The blue points towards the bottom of the left pane illustrate a single sample of events from the NHPP with intensity $\lambda$. The decision-maker selects $y=0.6$ and sweeps the sub-segment $[0,0.6]$, detecting each event therein with probability $\gamma(0.6)$. The red piecewise-constant function in the right pane illustrates the effective filtering probability over $[0,1]$. The points plotted in red then represent the events actually detected by the decision-maker during their imperfect search - which we observe are a subset of the events that actually arose.

In this paper, we consider a sequential variant of this problem, where the CIF, $\Lambda$, is unknown to the decision-maker, but the choice of endpoint $y$ can be updated over a series of rounds, in response to observing the locations of detected events in previous rounds. The decision-maker's aim is then to maximise the expected number of detected events over $T \in \mathbb{N}$ rounds. The study of this problem is motivated both by its theoretical challenge and its practical interest. 

Versions of this problem may arise in a number of settings such as ecological surveillance, defence, and logistics, where sightings of endangered species, criminal activity, or machine faults may for instance comprise the events of interest. As a motivating, and sufficiently general example, consider a scenario where observations are made by searchers (representing cameras, sensors, robotic and human searchers, etc.), that must restart at the same point after each round. We note that while in the material that follows we will treat the line segment as indexing space (for clarity and consistency), it could equivalently be thought of as indexing time or space-time and apply to a yet broader range of examples.

From a theoretical perspective, the problem is closely related to the one-dimensional case of the stochastic continuum-armed bandit (CAB) problem \citep{agrawal1995continuum}. This is a sequential decision-making problem where in each of a series of rounds $t \in [T] \equiv \{1,\dots,T\}$, a decision-maker selects an action $x_t \in [0,1]$ and receives a reward, which is a noisy realisation of some unknown smooth function $f:[0,1] \rightarrow [0,1]$ evaluated at $x_t$. The decision-maker's aim is to maximise the expected sum of rewards amassed over $T$ rounds. To realise this aim, the decision-maker must deploy a strategy which appropriately balances between exploring the action space $[0,1]$ to learn the function $f$, and exploiting this information, selecting actions known to produce larger rewards to maximise the cumulative total. 

In the Poisson process-based problem at hand, a similar dilemma arises, we lack knowledge of the filtered CIF - which corresponds to the reward function - and can only hope to maximise the sum of rewards by exploring the action space - i.e. choosing a range of endpoints $y \in [0,1]$. However, the feedback received on actions in our problem is much richer than in the standard CAB problem. In addition to a noisy realisation of the filtered CIF, $\Lambda\gamma$, we observe the location of detected events, which can help with the estimation of the reward function beyond the inferences from smoothness properties alone. Methods for the standard CAB problem are therefore inappropriate for the problem we face, as is the existing unmodified theory. In this paper we present a specific treatment of the previously described sequential endpoint selection problem, which we henceforth refer to as a Filtered Poisson Process Bandit (FPPB), deriving a bespoke decision-making algorithm and theoretical analysis of the problem.

\begin{figure}
\centering
\includegraphics[width=0.9\textwidth]{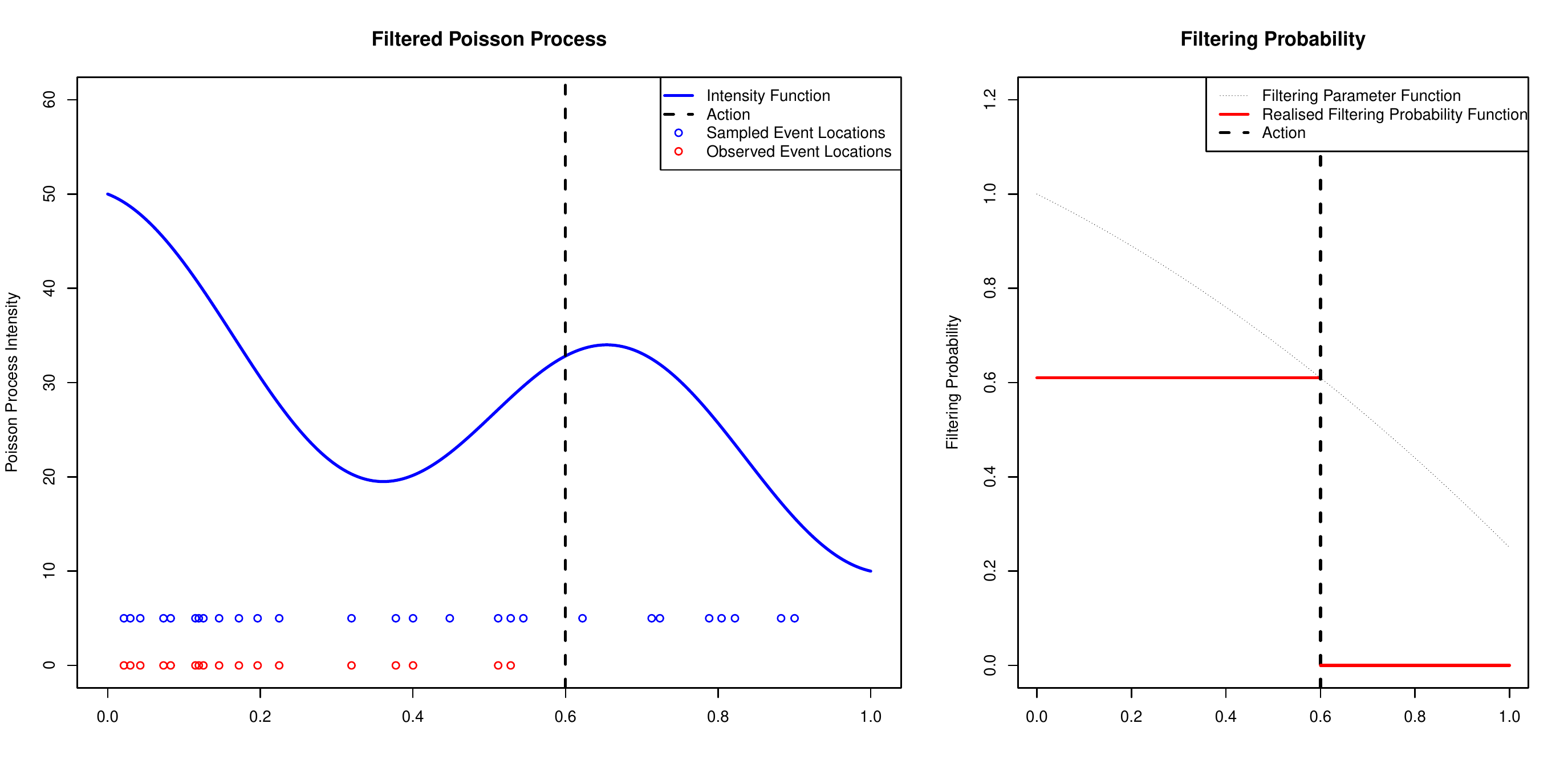}
\caption{Graphical representation of the filtering process.}
\label{fig::illustration}
\end{figure}

\subsection{Related Literature}
Sequential decision-making problems on continuous action spaces have been studied extensively, following from initial works of \cite{agrawal1995continuum} and \cite{kleinberg2005nearly}. Most successful strategies have employed a combination of adaptive discretisation of the action space, and optimism in the face of uncertainty. Our approach for the FPPB problem, also uses these techniques. 

Adaptive discretisation, as used in the ``Zooming" algorithm of \cite{kleinberg2008multi} and ``hierarchical online optimisation" (HOO) algorithm of \cite{bubeck2011x}, reduces the available action space in round $t$ to some $\mathcal{A}_t \subset [0,1]$. Restricting the action set ensures exploration occurs at a predictable rate, and makes the action selection more straightforward. Gradually, as the rounds proceed and more information is gathered, $\mathcal{A}_t$ is increased, usually in a data-adaptive fashion to permit choice from a more granular set of actions. Intuitively, this is also appealing, as when estimates of the reward are very crude, there is little motivation to make decisions at a very granular level.

Optimistic approaches are those which encourage an appropriate balance of exploration and exploitation by making decisions with respect to high probability upper confidence bounds (UCBs) on the expected reward of the available actions. The Zooming and HOO algorithms both calculate UCBs for the reward of available actions in each round and select the action with the largest UCB. These approaches were the first to achieve order optimal performance, in terms of regret, for this class of problems.

Strong results have also been obtained by approaches which use Gaussian processes and avoid discretisation of the action space. The GP-UCB (Gaussian Process - Upper Confidence Bound) algorithm of \cite{srinivas2010gaussian} constructs an upper confidence bound on the reward function over all actions, rather than at specific points, and selects the action which maximises this UCB function. This method also has order optimal performance guarantees, but with respect to a Bayesian measure of regret, rather than the frequentist one used in the analysis of the Zooming and HOO algorithms.

It is worth noting that none of these algorithms can sensibly be applied to the FPPB, and that their theoretical guarantees do not carry to the FPPB problem. Principally, this is because they lack a means to handle the additional feedback in terms of the location data, but a more subtle point is that without modification, these methods are not suited to unbounded rewards, as we have in this setting, with the Poisson distributed reward.

\cite{grant2020adaptive} consider a filtered Poisson bandit problem which is similar in some senses to ours, but theirs employs a fixed discretisation of the action space, such that the spatial locations of the events are irrelevant. They focus instead on the challenges of choosing multiple non-overlapping sub-segments and analyse performance with respect to the best possible action among a fixed discrete set. \cite{grant2019adaptive} considers a continuous action space, but without filtering of the observations. Inference is therefore more straightforward in this setting, and the Thompson Sampling method proposed is not applicable to the FPPB setting. Recently, \cite{lu2019optimal} provide an algorithm combining the adaptive discretisation of \cite{kleinberg2008multi} and heavy tailed UCBs of \cite{bubeck2013bandits} for a version of the CAB problem with heavy-tailed reward noise distributions. While the Poisson does fit in to this class of distributions, it also enjoys tighter bespoke concentration results, and a general heavy-tailed approach is overly conservative for the FPPB - even if event locations were not observed.

\subsection{Key Contributions and Structure}
The main contribution is a UCB algorithm with $\tilde O(T^{2/3})$ regret over $T$ rounds. By derivation of a lower bound, we show that under the assumptions on the CIF, this is optimal up to a logarithmic factor. From the methodological viewpoint, we extend the Lipschitz multi-armed bandit framework \citep{kleinberg2008multi} to deal with a filtered Poisson process on continuum. 

The remainder of the paper is structured as follows. In Section \ref{sec::model} we precisely state the problem of interest. In Section \ref{sec::algorithm} we present our UCB approach to the problem. Sections \ref{sec::uppbound} and \ref{sec::lowbound} provide the upper and lower bounds on regret respectively. We conclude with a simulation of our method in Section \ref{sec::experiments}, and discussion in Section \ref{sec::discussion}.

\section{Model} \label{sec::model}
The formal specification of the FPPB problem is as follows. In rounds $t\in[T]$, the decision-maker selects an endpoint $y_t \in [0,1]$ and makes an observation on the sub-segment $[0,y_t]$. The environment generates a realisation of the NHPP with CIF $\Lambda$, consisting of an increasing sequence of event locations $\{X_{t,1},X_{t,2},\ldots,X_{t,N_t}\} \in [0,1]^{N_t}$, where $ N_t\sim$ Poisson$(\Lambda(1))$. The end-point selected by the decision-maker implies a filtering probability $\gamma(y_t) \in [0,1]$, such that events to the left of $y_t$ are detected independently of each other with probability $\gamma(y_t)$, and all events to the right of $y_t$ are not detected. As a result, a sequence of i.i.d. Bernoulli($\gamma(y_t)$) random variables, $B_1,B_2,\ldots,B_{N_t}$,
is generated. The decision maker receives the count of detected events $R_t \equiv R_t(y_t) = \sum_{k=1}^{N_t} \mathbbm{1}(B_{t,k}=1, X_{t,k}\leq y_t)$ as a reward, and observes the locations of detected events $X_{t,k}$ with $B_{t,k}=1$ and $X_{t,k}\leq y_t$. By construction, $R_t\sim$ Poisson$(\Lambda(y_t)\gamma(y_t))$.

The decision-maker's objective is to maximise the sum of rewards obtained over $T$ rounds, $\sum_{t=1}^T R_t$. To realise this objective we aim to determine a \emph{policy}, $A$, which maps from a history of actions and observations to a next action, which maximises the expected reward, or equivalently minimises the \emph{regret}, \begin{equation}
    Reg_A(T) = \mathbb{E}\left(\sum_{t=1}^T R_t(z^*)-R_t(y_t) \right),
\end{equation} where $z^* \in \argmax_{y \in [0,1]} \Lambda(y)\gamma(y)$ is an optimal endpoint which maximises the expected per-round reward. Here the expectation is with respect to both the random process governing the generation and filtering of events and the decision-maker's actions. We will be interested in upper bounding the regret as a function of $T$ for our proposed algorithm, and comparing the order of this upper bound to that of lower bounds on the best achievable regret of any algorithm.

Bounded regret is achievable only if the reward function is suitably well-behaved as to admit learning from a finite sample of observations. This is ensured through assumptions on the form of the CIF and filtering function. These assumptions, enforced throughout the paper, are Lipschitz continuity of the filtered CIF and a rate bound,
\begin{align*}
\text{A1: }&|\gamma(y)\Lambda(y)-\gamma(x)\Lambda(x)|\leq m |y-x|, \forall x,y\in [0,1],\\
\text{A2: }&\lambda(y)\leq \lambda_{\max},
\end{align*}
for $m,\lambda_{\max}\geq 0$ known and finite. Assumptions A1--A2 are used to bound the estimation error for the expected number of detected events in each cell; this can be achieved by including in the cell index an additive term proportional to the cell length. We also assume that $\gamma_{\min}=\inf_{y\in(0,1]}\{\gamma(y)>0\}>0$; this is without loss of generality,  as segments with $\gamma(\cdot)=0$ do not contain the optimal endpoint.

\section{Algorithm} \label{sec::algorithm}
In this section we present our algorithm for the FPPB problem, CIF-UCB, given as Algorithm \ref{alg::SE}. 

\begin{algorithm}[hp]
   \caption{CIF-UCB (Cumulative Intensity Function - Upper Confidence Bound)}
   	\label{alg::SE}
\begin{algorithmic}[1]
   \STATE {\bfseries Input:} Rate bound $\lambda_{\max}$, filtering probabilities $\gamma(\cdot)$, Lipschitz constant $m$, active cell set $\mathcal A_1 = \{(0,1]\}$, effective number of samples $V_1(0,1)=\emptyset$, index $\mathcal I_1(0,1)=m$.
   \FOR{$t=1$ {\bfseries to} $T$}
	\STATE \underline{\emph{Selection Rule:}} 
	\STATE Find cell \begin{displaymath}
			(a_t,b_t]=\argmax_{(x,y]\in \mathcal A_{t}}\mathcal I_t(x,y),
			\end{displaymath}
			breaking ties randomly.
	\STATE Do a sweep up to $b_t$.
	\STATE  Update $V_{t+1}(a_t,b_t) = V_{t}(a_t,b_t)\cup \{t\}$, and 
		\begin{align*}
		\zeta_{t+1}(b_t)&=\frac{6\max\{1,\lambda_{\max}\}\log(T)}{\sum_{i=1}^{|V_{t+1}(a_t,b_t)|}\gamma(b_{\tau_i})}  +\sqrt{\frac{6\lambda_{\max}\log(T)}{\sum_{i=1}^{|V_{t+1}(a_t,b_t)|}\gamma(b_{\tau_i})}}.
		\end{align*}
		 
	\STATE Update $\bar \Lambda_{t+1}(b_t)$ as in (\ref{eq:las1}).
		 
	\STATE \underline{\emph{Division Rule:}}
		\IF {$m(b_t-a_t)\geq \zeta_{t+1}(b_t)$} 
		
		\STATE Update the active cell set $\mathcal A_{t+1}=\mathcal A_t \setminus \{(a_t,b_t]\} \cup \{(a_t,(a_t+b_t)/2],((a_t+b_t)/2,b_t]\}$. 
		 \STATE Set $V_{t+1}((a_t,(a_t+b_t)/2)=V_{t+1}(((a_t+b_t)/2,b_t)=V_{t+1}(a_t,b_t)$, and
		\begin{align*}
		\zeta_{t+1}\left(\frac{a_t+b_t}{2}\right) & = \frac{6\max\{1,\lambda_{\max}\}\log(T)}{\sum_{i=1}^{|V_{t+1}(a_t,(a_t+b_t)/2)|}\gamma(b_{\tau_i})}  +\sqrt{\frac{6\lambda_{\max}\log(T)}{\sum_{i=1}^{|V_{t+1}(a_t,(a_t+b_t)/2)|}\gamma(b_{\tau_i})}}, \\
		\zeta_{t+1}(b_t) & = \frac{6\max\{1,\lambda_{\max}\}\log(T)}{\sum_{i=1}^{|V_{t+1}((a_t+b_t)/2,b_t)|}\gamma(b_{\tau_i})}  +\sqrt{\frac{6\lambda_{\max}\log(T)}{\sum_{i=1}^{|V_{t+1}((a_t+b_t)/2,b_t)|}\gamma(b_{\tau_i})}}.
		\end{align*}
		
		\STATE Define $\bar\Lambda_{t+1}((a_t+b_t)/2)$ and $\bar\Lambda_{t+1}(b_t)$ as in (\ref{eq:las1}).
		\ENDIF
		\STATE \underline{\emph{UCB Computation:}}
	   \STATE Set $\mathcal I_{t+1}(x,y)=\gamma(y) \bar\Lambda_{t+1}(y)+m(y-x)+\gamma(y)\zeta_{t+1}(y)$ for all cells $(x,y]\in \mathcal A_{t+1}$.
   \ENDFOR
\end{algorithmic}
\end{algorithm}
At a high level, CIF-UCB proceeds as follows. For each round $t=1,\ldots,T$, the algorithm maintains a set of active cells, $\mathcal A_t$, which form a partition of $[0,1]$. An index, $\mathcal{I}_{t}$, taking the form of optimistic estimate of the expected reward, is computed for each cell in $\mathcal{A}_t$. The algorithm selects the right endpoint of the active cell with largest index as the action for that round. Initially, the active set contains the unit interval, $\mathcal A_1 = \{(0,1]\}$, so that the algorithm does a complete sweep in the first round. If the number of sweeps of a cell exceeds some threshold in relation to its length, the cell is split in half. Hence, active cells make up a partition of the interval $[0,1]$ for all rounds. A new cell inherits the  number of sweeps and detection count that fall in its interval from the parent cell. 

Accumulating rewards over the interval to the left of the selected endpoint makes the problem structure combinatorial in nature, which poses a challenge for the analysis. The insight that makes the analysis tractable is that, by the independent increment property of the Poisson process, the filtered Poisson counts corresponding to the active cells that lie to the left of the endpoint selected by the algorithm in each round are independent. This leads to a CIF estimator for each active cell with tight error bounds.

We complete the notation needed to define the CIF estimator. Let $\{\mathcal F_t\}_{t=1}^{T}$ be the filtration induced by the sequence of event locations and cell selections $((a_t,b_t])_{t=1}^T$. Also, let  $$V_t(x,y)=\{\tau_1,\tau_2, \ldots\} \subseteq [t]$$ be the collection of (random) times when active cell $(x,y]$ is swept by round $t$ and let, $$Z_{\tau_i}(y)=\sum_{k=1}^{N_{\tau_i}}\mathbbm{1}(B_{\tau_i,k}=1, X_{\tau_i,k}\leq y)$$ be the filtered Poisson count to the left of $y$ in round $\tau_i$. Finally, let $\sum_{i=1}^{|V_t(x,y)|} Z_{\tau_i}(y)$ be the total filtered Poisson count to the left of $y$ over the rounds when cell $(x,y]$ is swept. When the context is clear, we write $V$ in lieu of $V_t(x,y)$

For active cell $(x,y]$, $\Lambda(y)$ is estimated by dividing the cumulative filtered Poisson counts up to $y$ by its effective number of sweeps by round $t$,
\begin{equation}\label{eq:las1}
\bar\Lambda_t(y) =\frac{\sum_{i=1}^{|V_t(x,y)|} Z_{\tau_i}(y)}{\sum_{i=1}^{|V_t(x,y)|}\gamma(b_{\tau_i})}.
\end{equation}
Essentially, in (\ref{eq:las1}) the filtered Poisson count is \emph{unfiltered} by dividing it by $\sum_{i=1}^{|V_t(x,y)|}\gamma(b_{\tau_i})$. It's easy to see that $\bar\Lambda_t(y)$ is an unbiased estimator of $\Lambda(y)$. 

CIF-UCB samples from the origin to the endpoint of the active cell with largest index, and divides the latter cell if its length exceeds certain threshold. The complexity of the CIF-UCB is $O(T)$ for the variable updates, and $O(\sum_{t=1}^T t\log t)= O(T^{2}\log T)$ for sorting the indices, since there are at most $t$ active cells by round $t$.


\section{Upper Bound on Regret} \label{sec::uppbound}
In this section we present the first of our main theoretical contributions, an upper bound on the regret of CIF-UCB.

\begin{theorem} \label{thm:upperbound}
The regret of CIF-UCB applied to the FPPB problem, with CIF and filtering function satisfying Assumptions A1 and A2 satisfies $$Reg(T) =  \tilde O(T^{2/3}).$$

\end{theorem}

\begin{proof}
The proof has three main stages. We first bound the CIF estimator error for each active cell (Lemma \ref{lemma:1}), and then use the Lipschitz assumption to extend the bound to include all the points inside an active cell (knowing that one of these points is an optimal endpoint for some active cell; Corollary \ref{lemma:concbound1}). Second, we use the Division rule to express the confidence bound of each active cell in terms of its length (Lemma \ref{lemma:important}), which yields a bound for the per-round regret of the cell selected by the algorithm (Lemma \ref{lemma:important2}). Finally, we accumulate these per-round regrets to obtain an upper bound for the regret over $T$ rounds.

Firstly, we present the following concentration result, which asserts that the difference between the true CIF and the estimated CIF is unlikely to exceed the upper confidence terms used in Algorithm \ref{alg::SE}.

\begin{lemma}\label{lemma:1} Let $(x,y]$ be an active cell in round $t$. Then, 
\[
P\left(\left| \bar\Lambda_t(y)-\Lambda(y)\right|>\zeta_t(y)\right)  \leq 2 T^{-2},
\]
where
\[
\zeta_t(y)= \frac{6\log(T)\max\{1,\lambda_{\max}\}}{\sum_{i=1}^{|V_t(x,y)|}\gamma(b_{\tau_i})}+\sqrt{\frac{6\lambda_{\max}\log(T)}{\sum_{i=1}^{|V_t(x,y)|}\gamma(b_{\tau_i})}}.
\]
\end{lemma}

\begin{proof}
The Poisson count $Z_{\tau_i}(y)$ is $\mathcal F_{\tau_i}$ measurable and,
\[
E[Z_{\tau_i}(y)| \mathcal F_{\tau_{i-1}}] = \Lambda(y) \gamma(b_{\tau_i}), \text{ a.s.}
\]
Defining,
\[
M_k(y)=\sum_{i=1}^{k} (Z_{\tau_i} (y) - \Lambda(y) \gamma(b_{\tau_i})),
\]
it follows that $\{M_{k\wedge |V|}(y), \mathcal F_{\tau_k}\}_{k\geq 1}$ is a martingale, and $M_{k\wedge |V|}(y) - M_{(k-1)\wedge |V|}(y) =  (Z_{\tau_k} (y) - \Lambda(y) \gamma(b_{\tau_i}))\mathbbm{1}(k\leq |V|)$ is a martingale difference sequence. By Lemma 1 in \citep{grant2020adaptive}, 
\begin{align*}
P\left(\sum_{i=1}^{k\wedge |V|} (Z_{\tau_i} (y) - \Lambda(y) \gamma(b_{\tau_i})) >\eta\right) \leq \exp\left(-\frac{\eta^2}{2\Lambda(y)\sum_{i=1}^{|V|}\gamma(b_{\tau_i})+2\max\{1,\Lambda(y)\}\eta}\right).
\end{align*}
Solving for the r.h.s. above equal to $T^{-3}$ leads to,
%
%
\begin{align*}
\eta&=3\log(T)\max\{1,\Lambda(y)\} +\sqrt{(3\log(T)\max\{1,\Lambda(y)\})^2+6\Lambda(y)\log(T)\sum_{i=1}^{|V|}\gamma(b_{\tau_i})}\\
&\leq 6\log(T)\max\{1,\lambda_{\max}\}+\sqrt{6\lambda_{\max}\log(T)\sum_{i=1}^{|V|}\gamma(b_{\tau_i})}.
\end{align*}
It follows that the probability that
\begin{align*}
\sum_{i=1}^{k\wedge |V|} (Z_{\tau_i} (y) - \Lambda(y) \gamma(b_{\tau_i})) > 6\log(T)\max\{1,\lambda_{\max}\}  +\sqrt{6\lambda_{\max}\log(T)\sum_{i=1}^{|V|}\gamma(b_{\tau_i})},
\end{align*}
is at most $T^{-3}$ for each $k\leq T$. Taking a union bound over all $k\leq T$, and replacing for the definition of $\bar\Lambda_t(y)$ and $\zeta_t(y)$ results in
\[
P\left(\bar\Lambda_t(y)-\Lambda(y)>\zeta_t(y)\right)  \leq T^{-2}.
\]

Finally, using the same approach it can be shown that
\[
P\left(\bar\Lambda_t(y)-\Lambda(y)<-\zeta_t(y)\right)  \leq T^{-2},
\]
so the proof is complete.

\end{proof}

The Lipschitz assumption can be used to extend this to a high probability bound on the filtered CIF for active cells.

\begin{corollary}\label{lemma:concbound1}
Let $(x,y]\in \mathcal A_t$. Then, with probability at least $1-2T^{-2}$
\begin{align*}
&\sup_{x<c\leq y}\left|\gamma(y)\bar\Lambda_{t}(y) - \gamma(c)\Lambda(c) \right|\leq m(y-x) +\gamma(y) \zeta_t(y).
\end{align*}

\end{corollary}

\begin{proof}
By the Lipschitz assumption,
\[
\sup_{x<c\leq y}\left|\gamma(y)\Lambda(y) - \gamma(c)\Lambda(c) \right|< m(y-x).
\]
Hence,
\begin{align*}
P\big(\sup_{x<c\leq y}\left|\gamma(y)\bar\Lambda_{t}(y) - \gamma(c)\Lambda(c) \right|>m(y-x)+\gamma(y)\zeta_t(y)\big) \leq P(|\bar\Lambda_t(y)-\Lambda(y)|>\zeta_t(y)).
\end{align*}
\end{proof}

The index of a cell $(x,y]$ active in round $t$ is
\[
\mathcal I_t(x,y)=\gamma(y) \bar\Lambda_{t}(y)+m(y-x)+\gamma(y)\zeta_t(y).
\]
The $\gamma(y) \bar\Lambda_{t}(y)$ part of the index induces exploitation, while the $m(y-x)+\gamma(y)\zeta_t(y)$ term promotes exploration.

All the results that follow in this section are on the sample paths where 
\begin{equation}\label{eq:cc1}
\sup_{x<c\leq y}\left|\gamma(y)\bar\Lambda_{t}(y) - \gamma(c)\Lambda(c) \right|\leq m(y-x)+\gamma(y) \zeta_t(y)
\end{equation}
holds for all rounds $t=1,\ldots, T$. By Corollary \ref{lemma:concbound1}, the contribution to the regret of the sample paths that violate (\ref{eq:cc1}) is of order $O(1)$, after accounting for the $T$ rounds and up to $T$ cells by round $T$.

Our next result bounds the upper confidence term $\zeta_t$ for an active cell on the high probability event of Corollary \ref{lemma:concbound1}.

\begin{lemma}\label{lemma:important}
For $(x,y]\in\mathcal A_t$,
\[
 \zeta_t(y)\leq 4m^2(y-x)^2 \max\{1,1/\lambda_{\max}\}+2m(y-x).
\]
\end{lemma}

\begin{proof}
Let $V^{(p)}(x,y)$ be the set of rounds the parent cell of $(x,y]$ got swept. The Division rule for the parent cell implies
\[
2m(y-x)\geq \frac{6\max\{1,\lambda_{\max}\}\log(T)}{\sum_{i=1}^{|V^{(p)}(x,y)|}\gamma(b_{\tau_i})}+\sqrt{\frac{6\lambda_{\max}\log(T)}{\sum_{i=1}^{|V^{(p)}(x,y)|}\gamma(b_{\tau_i})}}.
\]
Then, we obtain the conservative lower bound,
\begin{equation}\label{eq:ll1}
\sum_{i=1}^{|V^{(p)}(x,y)|}\gamma(b_{\tau_i})\geq \frac{3\lambda_{\max}\log(T)}{2m^2(y-x)^2}.
\end{equation} 
Next we upper bound $\zeta_t(y)$,
\begin{align*}
\zeta_t(y)&\leq\frac{6\max\{1,\lambda_{\max}\}\log(T)}{\sum_{i=1}^{|V^{(p)}(x,y)|}\gamma(b_{\tau_i})}+\sqrt{\frac{6\lambda_{\max}\log(T)}{\sum_{i=1}^{|V^{(p)}(x,y)|}\gamma(b_{\tau_i})}}\\
&\leq 4m^2(y-x)^2 \max\{1,1/\lambda_{\max}\}+2m(y-x),
\end{align*} where the first inequality follows from  the definition of $\zeta_t(y)$, and the second inequality follows from the lower bound \eqref{eq:ll1}.

\end{proof}

Let $z^*$ be an optimal endpoint (i.e., $\gamma(z^*)\Lambda(z^*)\geq \gamma(y)\Lambda(y)$ for all $y\in[0,1]$), and $(u_t,v_t]\in \mathcal A_t$ the cell that contains $z^*$ in round $t$. The next result bounds the  regret $$\Delta(a_t,b_t)=\gamma(z^*)\Lambda(z^*)-\gamma(b_t)\Lambda(b_t),$$ incurred in each round in terms of the length of the cell selected by the algorithm.

\begin{lemma}\label{lemma:important2}
The round $t$ regret $\Delta(a_t,b_t)$ satisfies
\[
\Delta(a_t,b_t)\leq 8m^2(b_t-a_t)^2 \max\{1,1/\lambda_{\max}\}+5m(b_t-a_t).
\]
\end{lemma}

\begin{proof}
We will show that 
\begin{align*}
\gamma(z^*)\Lambda(z^*) &\leq \mathcal I_t(a_t,b_t) \leq \gamma(b_t)\Lambda(b_t)+ 5m(b_t-a_t)  +8m^2(b_t-a_t)^2\max\{1,1/\lambda_{\max}\}, 
\end{align*}
from where the claim follows.

For the first inequality, we observe that,
\begin{align*}
\mathcal I_t(a_t,b_t) \geq \mathcal I_t(u_t,v_t) &\geq  \gamma(v_t) \Lambda(v_t)+m(v_t-u_t) \\
&\geq \gamma(v_t)\Lambda(v_t) + m(v_t -z^*) \geq \gamma(z^*)\Lambda(z^*). 
\end{align*} In order, these inequalities follow from the Selection rule, the definition of the index function $\mathcal{I}_t$ and Corollary \ref{lemma:concbound1}, the fact that $z^* \in (u_t,v_t]$, and the Lipschitz assumption. In the other direction, we have by application of Corollary \ref{lemma:concbound1}, and then Lemma \ref{lemma:important},
\begin{align*}
\mathcal I_t(a_t,b_t) &\leq \gamma(b_t)\Lambda(b_t)+m(b_t-a_t)+2\gamma(b_t)\zeta_t(b_t) \\
&\leq \gamma(b_t)\Lambda(b_t)+5m(b_t-a_t)  +8m^2(b_t-a_t)^2 \max\{1,1/\lambda_{\max}\}.
\end{align*}
\end{proof}

The final stage of the proof combines these results to realise the bound on regret. By Lemma \ref{lemma:important2}, the regret of cells with length at most $\ell$  is bounded by 
\begin{equation}\label{eq:dd1}
T  (8m^2\ell^2 \max\{1,1/\lambda_{\max}\}+5m \ell)
\end{equation}
over all rounds.

Cells with final length $\ell$ have three properties: (i)  there are at most $1/\ell$ such cells; (ii) their regret per round is at most $8m^2\ell^2 \max\{1,1/\lambda_{\max}\}+5m\ell$ (Lemma \ref{lemma:important2}); and (iii), satisfy (Division rule)
\[
m\ell\leq \frac{6\log(T)\max\{1,\lambda_{\max}\}}{\sum_{i=1}^{|V|}\gamma(b_{\tau_i})}+\sqrt{\frac{6\lambda_{\max}\log(T)}{\sum_{i=1}^{|V|}\gamma(b_{\tau_i})}}.
\]
Solving the quadratic inequality leads, after some algebra, to 
\[
\sum_{i=1}^{|V|}\gamma(b_{\tau_i})\leq \frac{12\max\{1,\lambda_{\max}\}\log(T)}{m \ell}+\frac{6\lambda_{\max}\log(T)}{m^2 \ell^2}.
\]
Since $|V|\gamma_{\min}\leq\sum_{i=1}^{|V|}\gamma(b_{\tau_i})$, the number of selections is bounded above by
\[
|V|\leq \frac{12\max\{1,\lambda_{\max}\}\log(T)}{\gamma_{\min}m \ell}+\frac{6\lambda_{\max}\log(T)}{\gamma_{\min}m^2 \ell^2}.
\]
Hence, the total regret from cells of length $\ell$ is at most
\begin{align}
&~ ~ ~ (8m^2\ell \max\{1,1/\lambda_{\max}\}+5m) |V| \nonumber \\
&\leq (8m^2\ell \max\{1,1/\lambda_{\max}\}+5m) \left(\frac{12\max\{1,\lambda_{\max}\}\log(T)}{\gamma_{\min}m \ell}+\frac{6\lambda_{\max}\log(T)}{\gamma_{\min}m^2 \ell^2}\right)\nonumber\\
&=\frac{\log(T)}{\gamma_{\min}}\bigg(96 m \max\{\lambda_{\max},1/\lambda_{\max}\}  +\frac{108\max\{1,\lambda_{\max}\}}{\ell}+\frac{30\lambda_{\max}}{m\ell^2}\bigg).\label{eq:ll2}
\end{align}

Using Eqs. (\ref{eq:dd1}) and (\ref{eq:ll2}) with $\ell=2^{-k}$ results in,
\begin{align}
Reg(T) &\leq T  (8m^2 4^{-k} \max\{1,1/\lambda_{\max}\}+5m 2^{-k})\nonumber\\
&\enspace ~ +\frac{\log(T)}{\gamma_{\min}}\bigg(96 m \max\{\lambda_{\max},1/\lambda_{\max}\}  +108\max\{1,\lambda_{\max}\}\sum_{i=0}^{k-1}2^{i}+\frac{30\lambda_{\max}}{m}\sum_{i=0}^{k-1}4^{i}\bigg)\nonumber\\
&\leq T  (8m^2 4^{-k} \max\{1,1/\lambda_{\max}\}+5m 2^{-k})\nonumber\\
&\enspace ~ +\frac{\log(T)}{\gamma_{\min}}\bigg(96 m \max\{\lambda_{\max},1/\lambda_{\max}\}  +36\max\{1,\lambda_{\max}\}2^{k}+\frac{10\lambda_{\max}}{m}4^{k}\bigg).\label{eq:ult1}
\end{align}
for all integer $k\geq 0$. The value of $k$ that minimises regret equalises the leading growth rates of both summands in (\ref{eq:ult1}), meaning that $2^k=T^{1/3}$. The claim follows from here.

\end{proof}


\section{Lower Bound on Regret} \label{sec::lowbound}
In this section we give a lower bound on the regret obtained by any algorithm for the filtered Poisson process bandit. The result is given below as Theorem \ref{thm::LB}, and we see, subject to further minor conditions on the filtering function, that the order of the lower bound on regret matches that of the upper bound on the regret of CIF-UCB up to a logarithmic factor. In this sense, CIF-UCB is therefore asymptotically order optimal (up to the exclusion of logarithmic factors).

\begin{theorem} \label{thm::LB}
For the filtered Poisson process bandit problem on $[0,1]$ as described in Section \ref{sec::model} with filtering function $\gamma$ satisfying 
\begin{equation}
    \frac{\gamma(a)-\gamma(b)}{b-a} \geq \frac{1}{4} \gamma\bigg(\frac{a+b}{2}\bigg) \label{eq::gammaLB}
    \end{equation} for any $0 \leq a \leq b \leq 1$, there exists a valid CIF such that the regret of any algorithm is bounded below as \begin{equation*}
Reg(T) = \Omega(T^{2/3}).
\end{equation*} 
\end{theorem}

The proof of this lower bound is based on an established analytical technique of relating the regret of an algorithm for a continuum armed bandit problem to that of an algorithm for an associated finite-armed bandit problem. A lower bound on regret for the finite-armed problem is then utilised to lower bound the regret of the continuum armed bandit algorithm.

Here, such an associated finite-armed bandit problem must share the filtering structure of the FPPB to relate regret across the problems, and as such we require a bespoke finite-armed problem. Therefore, before giving the proof of Theorem \ref{thm::LB}, we introduce a \emph{filtered Poisson multi-armed bandit} (FPMAB) problem which can be viewed as a discretised version of the FPPB. We derive a lower bound on the regret of any algorithm for the FPMAB, which is a key component of the proof of Theorem \ref{thm::LB}.

We define the FPMAB problem as follows. The problem is instantiated by $K$ arms with mean parameters $\mu_k \in [0,\lambda_m]$. Each mean parameter may be decomposed as the product of a CIF parameter $\Lambda_k \in [0,\lambda_m]$ and filtering parameter $\gamma_k \in [0,1]$ - i.e. $\mu_k=\Lambda_k\gamma_k$, $k \in [K]$. The ordered CIF parameters comprise a monotonically increasing sequence, $\Lambda_1 \leq \Lambda_2 \leq \dots \leq \Lambda_K$, and the ordered filtering parameters comprise a monotonically decreasing sequence, $\gamma_1 \geq \gamma_2 \geq \dots \geq \gamma_K$. 

The problem takes place over a series of rounds $t \in [T]$, in each of which the decision-maker selects an arm $a_t \in [K]$ and receives a stochastic reward $R_t=R(a_t)$. In addition, the decision-maker observes \emph{filtered observations}, $\tilde{R}_{k,t}$ for $1 \leq k \leq a_t$. These observations are distributed as $$\tilde{R}_{k,t} \sim  \text{Poisson}(\gamma_{a_t}(\Lambda_k-\Lambda_{k-1})).$$ The reward is defined as the sum of the filtered observations $R_t= \sum_{k=1}^{a_t} \tilde{R}_{k,t}$, and therefore follows a Poisson distribution with parameter $\mu_a$, by the superposition property of the Poisson distribution.

Similarly as to the FPPB, the decision-maker's aim is to minimise regret in $T$ rounds, defined as $$Reg(T) = \mathbb{E}\bigg(\sum_{t=1}^T R_t(a^*)-R_t(a_t)\bigg),$$ where $a^* \in \argmax_{k \in [K]} \mu_k$ is an optimal arm. We have the following minimax lower bound on the regret of any algorithm for the FPMAB problem. 

\begin{theorem} \label{thm::MABLB} 
For any number of arms $K \geq 2$, horizon $T \in \mathbb{N}$, a set of filtering parameters $\gamma_1,\dots,\gamma_K$ satisfiying \begin{equation}
\gamma_k \geq \Big(1+\epsilon\Big)\gamma_{k+1} \label{eq::disc_gamma_condition}
\end{equation} for $k \in [K-1]$, and some small $\epsilon>0$ there exist a set of CIF parameters $\Lambda_1,\dots,\Lambda_K$ and a known constant $C>0$ such that the regret of any algorithm for the FPMAB problem is at least \begin{equation}
    C\epsilon\left(T -\frac{T}{K} - \frac{T}{2}\sqrt{\frac{3\epsilon^2T}{K}}\right). \label{eq::LBMAB}
\end{equation}
\end{theorem}

This Theorem is similar in spirit to the lower bound on regret for stochastic multi-armed bandits with bounded rewards in Theorem 5.1 of \cite{auer2002nonstochastic}, and its generalisation in \cite{bubeck2011lipschitz}. Indeed Theorem \ref{thm::MABLB} has the same order with respect to $\epsilon$ and $T$ however there are key differences in the proof of the result. Firstly, Theorem \ref{thm::MABLB} considers filtered Poisson random variables, and therefore parts of the analysis are specific to the KL divergence between two Poisson random variables rather than Bernoulli random variables in the bounded case. Secondly, here we deal with the additional combinatorial feedback of FPMAB problem, and require further analyses to handle the resulting complexities. 

In the remainder of this section we prove Theorems \ref{thm::LB} and \ref{thm::MABLB}.

\subsection{Proof of Theorem \ref{thm::LB}}

\begin{proof}
Consider the instance of the filtered Poisson process bandit problem referred to as $\mathcal{I}(x^*,\epsilon)$, for $x^*\in [0,1]$ and $\epsilon>0$, and specified by the following reward function \begin{equation}
\nu_{x^*,\epsilon}(x) = \begin{cases}&m\epsilon(1+ \epsilon - |x-x^*|), \quad x: |x-x^*| \leq \epsilon \\
&\min(mx,m\epsilon), \quad \quad \quad \quad \quad \quad \quad \quad \text{othw.} \end{cases}
\end{equation} Such a reward function is realised by setting the CIF to \begin{equation}
\Lambda_{x^*,\epsilon}(x) = \begin{cases} &(\gamma(x))^{-1}m\epsilon\big[1+\epsilon-(x^*-x) \big], \quad x \in [x^*-\epsilon,x^*) \\
&(\gamma(x))^{-1}m\epsilon\big[1+\epsilon-(x-x^*)\big],\quad x \in [x^*,x+\epsilon) \\
&(\gamma(x))^{-1}\min(mx,m\epsilon), \quad \quad \quad \quad \quad \quad \quad \text{othw.} \end{cases}
\end{equation} 

To verify that this CIF is increasing, consider the derivative, \begin{displaymath}
\frac{d\Lambda_{x^*,\epsilon}(x)}{dx} = \begin{cases} &\frac{d(\frac{1}{\gamma})}{dx}m\epsilon\big[1+\epsilon-x^*+x\big]+m\epsilon(\gamma(x))^{-1}, \quad x \in [x^*-\epsilon,x^*) \\
&\frac{d(\frac{1}{\gamma})}{dx}m\epsilon\big[1+\epsilon+x^*-x\big]-m\epsilon(\gamma(x))^{-1}, \quad x \in [x^*,x^*+\epsilon) \\
&\frac{d(\frac{1}{\gamma})}{dx}mx + m(\gamma(x))^{-1}, \quad \quad \quad \quad \quad \quad \quad \quad \enspace x \in [0,\epsilon) \\
&\frac{d(\frac{1}{\gamma})}{dx}m\epsilon \quad \quad \quad \quad \quad \quad \quad \quad \quad \quad \quad \quad \quad \quad \quad \quad \enspace \text{othw.}
 \end{cases}
\end{displaymath}
We note that $(\gamma(x))^{-1}>1$ for all $x \in [0,1]$ since $\gamma: [0,1] \rightarrow [0,1]$, and that $d(\gamma(x))^{-1}/dx \geq 0$ for all $x \in [0,1]$ since $\gamma$ is assumed to be strictly increasing on $[0,1]$. It follows that for $x \in [x^*-\epsilon,x^*)$, \begin{align*}
    \frac{d\Lambda_{x^*,\epsilon}(x)}{dx} \geq \frac{d(\gamma(x))^{-1}}{dx}m\epsilon\bigg[1+\epsilon-\epsilon \bigg] + m \epsilon(\gamma(x))^{-1} = \frac{d(\gamma(x))^{-1}}{dx}m\epsilon + m\epsilon(\gamma(x))^{-1} > 0.
\end{align*} For $x \in [x^*,x^*+\epsilon)$, consider \begin{align}
    \frac{d\Lambda_{x^*,\epsilon}(x)}{dx} \geq \frac{d(\gamma(x))^{-1}}{dx}m\epsilon\bigg[1+\epsilon-\epsilon \bigg] - m \epsilon(\gamma(x))^{-1} = m\epsilon\bigg(\frac{d(\gamma(x))^{-1}}{dx} - (\gamma(x))^{-1}\bigg). \label{eq::derivonright}
\end{align} In the limit as $b-a \rightarrow 0$ condition \eqref{eq::gammaLB} implies that $-\frac{d\gamma(x)}{dx} \geq \gamma(x)$. We have, for a differentiable function $f$ such that $f(x)\neq 0$ that the derivative of $g(x)=1/f(x)$, that \begin{displaymath}
\frac{dg(x)}{dx} = \frac{-\frac{df(x)}{dx}}{(f(x))^2}.
\end{displaymath} Thus, \begin{displaymath}
\frac{d(\gamma(x))^{-1}}{dx} = \frac{-\frac{d\gamma(x)}{dx}}{(\gamma(x))^2} \geq \frac{-\gamma(x)}{(\gamma(x))^2} = (\gamma(x))^{-1},
\end{displaymath} and it follows from \eqref{eq::derivonright} that $d\Lambda_{x^*,\epsilon}(x)/dx >0$ for $x \in [x^*,x^*+\epsilon)$. For all other values of $x \in [0,1]$ it should be obvious that the derivative of the CIF is positive since it comprises a sum of non-negative terms. As such $\Lambda_{x^*,\epsilon}$ satisfies the necessary increasing assumption, and the instance  $\mathcal{I}(x^*,\epsilon)$ is a valid instance of the FPPB.

We will lower bound the regret of any algorithm for the problem instance $\mathcal{I}(x^*,\epsilon)$ by relating it to an instance of the filtered Poisson MAB problem.

We fix $K \in \mathbb{N}$ to be defined later and let $\epsilon=(2K)^{-1}$. Further we introduce the function $f_{\epsilon}:[K] \rightarrow [0,1]$ with $$f_{\epsilon}(a)=(2a-1)\epsilon, \quad a \in [K].$$   This function is used to map between actions in the MAB problem and the CAB problem. We then define an instance $\mathcal{J}(a^*,\epsilon)$ of the $K$-armed filtered Poisson MAB problem as that with arm means $$\mu_a= \nu_{x^*,\epsilon}(f_{\epsilon}(a)), \quad a \in [K],$$ and filtering probabilities $$\gamma_a = \gamma\bigg(\frac{2a-1}{K}\bigg), \quad a \in [K].$$ It follows that in the problem instance $\mathcal{J}(a^*,\epsilon)$ there is a single optimal arm $a^*\in [K]: x^* \in [\frac{a-1}{K},\frac{a}{K}]$ with expected reward $\mu_{a^*}=m\epsilon(1+\epsilon)$ and all other arms, $a \neq a^*$, have expected reward $\mu_a=m\epsilon$.

Let \texttt{ALG} be any algorithm for the CAB problem $\mathcal{I}(x^*,\epsilon)$. We will define \texttt{ALG'} as an associated algorithm for the MAB problem $\mathcal{J}(a^*,\epsilon)$. These algorithms are related as follows. When \texttt{ALG} selects an action $x_t \in [0,1]$, \texttt{ALG'} selects an arm $a_t \equiv a(x_t) \in [K]$ such that $$x_t \in \bigg( f_{\epsilon}(a_t) - \frac{1}{2K}, f_{\epsilon}(a_t)+ \frac{1}{2K}\bigg).$$ 

By definition of the FPMAB, \texttt{ALG'} will receive reward $R'(a_t) \sim Pois(\mu_{a_t})$ and per-arm observations $\tilde{R}'_{i,t} \sim Pois(\gamma(a_t)(\Lambda_i-\Lambda_{i-1}))$ for $i \leq a_t$. Similarly, \texttt{ALG} will receive reward $R(x_t) \sim Pois(\nu_{x^*,\epsilon}(x_t))$ and observe point data in $[0,x_t]$ derived from the filtered Poisson process. We shall also, however, demonstrate that $R(x_t)$ can be shown to have the same distribution as a certain probabilistic function of $\tilde{R}'(a_t)$ and use this representation to relate the regret of \texttt{ALG} and \texttt{ALG'}.

Define $Z$ to be a Poisson random variable with parameter $m\epsilon(1+\epsilon)$, and $Y$ to be a Poisson random variable with parameter $m\epsilon$. Then define $r_{x}$, a random variable whose distribution depends on $x \in [0,1]$, as follows, \begin{equation}
    r_{x}\equiv \begin{cases}
    &Z, \quad \text{ with probability } p_x \\
    &Y, \quad \text{ othw.}
    \end{cases}
\end{equation} where \begin{equation}
    p_x = \frac{1-\nu_{x^*,\epsilon}(x)}{1-\mu_{a(x)}}.
\end{equation}

It follows that \begin{align*}
    \mathbb{E}(r_x|x)&=m\epsilon\bigg((1-p_x)\mathbb{E}(Y) + p_x\mathbb{E}(Z)\bigg) \\
    &= m\epsilon\Bigg( 1 - \frac{1-\nu_{x^*,\epsilon}(x)}{1-\mu_{a(x)}} + \frac{1-\nu_{x^*,\epsilon}(x)}{1-\mu_{a(x)}} (1+\epsilon)\Bigg) \\
    &= m\epsilon \Bigg( 1+ \epsilon \enspace \frac{1-\nu_{x^*,\epsilon}(x)}{1-\mu_{a(x)}}\Bigg) \\
    &=\begin{cases} &m\epsilon\big(1 + \frac{\epsilon}{1-1-\epsilon}(1-1-\epsilon +m|x-x^*|)\big), \quad x: |x^*-x| \leq \epsilon \\
                   &m\epsilon, \quad \quad \text{othw.}
                   \end{cases}  \\
                   &=\mathbb{E}(R(x_t)).
\end{align*}

We notice that for both $\mathcal{I}(x^*,\epsilon)$ and $\mathcal{J}(a^*,\epsilon)$ the reward of the optimal action is $m\epsilon(1+\epsilon)$. Further we have that $\mathbb{E}(R(x_t)) \leq \mathbb{E}(R'(a(x_t)))$ for all $x_t \in [0,1]$. It therefore follows that the regret of \texttt{ALG'} serves as a lower bound on the regret of \texttt{ALG}, i.e. we have $$\mathbb{E}(Reg_{ALG}(T)) \geq \mathbb{E}(Reg'_{ALG'}(T)).$$  As \texttt{ALG'} is an algorithm for the FPMAB problem, its regret is lower bounded as in Theorem \ref{thm::MABLB}, and we therefore have $$\mathbb{E}(Reg_{ALG}(T)) \geq C\epsilon\Bigg(T -\frac{T}{K} - \frac{T}{2}\sqrt{\frac{3\epsilon^2T}{K}}\Bigg),$$ for a known constant $C>0$.

We complete the proof of Theorem \ref{thm::LB} by optimising our choice of $K$ as a function of $T$. Substituting $\epsilon=1/2K$, we have \begin{align*}
    \mathbb{E}(Reg_{ALG}(T)) \geq \frac{CT}{2K} - \frac{CT}{2K^2} - \frac{CT}{4K}\sqrt{\frac{3T}{4K^3}},
\end{align*} and choosing $K=O(T^{1/3})$ yields the stated result.

\end{proof}

\subsection{Proof of Theorem \ref{thm::MABLB}}

\begin{proof}
Given a set of filtering parameters $\gamma_1,\dots,\gamma_K$ we construct a problem instance where there is a single ``good" arm, $i \in [K]$, with mean reward $\mu_i=1+\epsilon$, for small $\epsilon \in (0,1/2]$, and all other arms, $k \neq i$, have mean rewards $\mu_k=1$. This is achieved by setting the CIF parameters as follows \begin{displaymath}
\Lambda^{(i)}_i= \frac{1+\epsilon}{\gamma_i}, \quad \Lambda^{(i)}_k= \frac{1}{\gamma_k}, ~ \forall k \neq i.
\end{displaymath} Here the superscript $\cdot^{(i)}$ denotes that $i$ is the good arm under this choice of parameters, and we notice that the condition of the filtering parameters \eqref{eq::disc_gamma_condition} is required for $\Lambda^{(i)}_1,\dots \Lambda^{(i)}_K$ to constitute a valid (i.e. increasing) sequence of CIF parameters.

We define three notions of probability and expectation, relevant to the analysis of problem instances of this type. Let $\mathbb{P}_*(\cdot)$ denote probability with respect to the above construction of the FPMAB where the good arm is chosen uniformly at random from $[K]$. Let $\mathbb{P}_i(\cdot)$ be defined similarly, but denote probability conditioned on the event that $i \in [K]$ is the good arm. Finally let $\mathbb{P}_{equ}$ denote probability with respect to a version where $\mu_k=1$ for all $k \in [K]$. We let $\mathbb{E}_*(\cdot)$, $\mathbb{E}_i(\cdot)$, and $\mathbb{E}_{equ}(\cdot)$ be respective associated expectation operators.

Let $A$ be the decision-maker's algorithm, let $$\mathbf{r}_t=(R(a_1),\dots,R(a_t))$$ denote the sequence of observed rewards in $t$ rounds, and $$\tilde{\mathbf{r}}_t=\Big( (\tilde{R}_{1,1},\dots,\tilde{R}_{a_1,1}),\dots,(\tilde{R}_{1,t},\dots,\tilde{R}_{a_t,t})\Big)$$ denote the sequence of filtered observations in $t$ rounds. Any algorithm $A$ may then be thought of a deterministic function from $\{\mathbf{r}_{t-1} ,\tilde{\mathbf{r}}_{t-1}\}$ to $a_t$ for all $t \in [T]$. Even an algorithm with randomised action selection can be viewed as deterministic, by treating a given run as a single member of the population of all possible instances of that algorithm.  

Further, we define $G_A = \sum_{t=1}^T R_t$ to be the reward accumulated by $A$ in $T$ rounds and $G_{max}=\max_{j \in [K]} \sum_{t=1}^T R_t(j)$ to be the reward accumulated by playing the best action.  The regret of $A$ in $T$ rounds may be expressed as $$Reg_A(T) = \mathbb{E}\big(G_{max}-G_A\big).$$

Let $N_k$ be the number of times an arm $k \in [K]$ is chosen by $A$ in $T$ rounds. The first step of the proof is to bound the difference in the expectation of $N_i$ when measured using $\mathbb{E}_i$ and $\mathbb{E}_{equ}$, i.e. to bound the difference in the number of times an algorithm with play $i$ between when $i$ is the good arm and when all arms are equally valuable. 

\begin{lemma} \label{lem::playsbound}
For any arm $i$ there exists a constant $C(\gamma_{i-1},\gamma_i,\gamma_{i+1})>0$ such that we have $$\mathbb{E}_i(N_i) \leq \mathbb{E}_{equ}(N_i) + \frac{T}{2}\sqrt{2\epsilon^2 \left(\mathbb{E}_{equ}(N_i)\frac{\gamma_{i-1}}{2(\gamma_{i-1}-\gamma_i)} + F_i\right) }$$ where \begin{equation}
F_i= C(\gamma_{i-1},\gamma_i,\gamma_{i+1})\sum_{k=i+1}^K \gamma_k\mathbb{E}_{equ}(N_k),
\end{equation}  for $\epsilon \leq \frac{\gamma_i}{2\gamma_{i+1}}-\frac{\gamma_i}{2\gamma_{i-1}}$, and $C(\gamma_{i-1},\gamma_i,\gamma_{i+1})$ is a known positive constant.
\end{lemma}

By construction of the CIF paramters $\Lambda_1^{(i)},\dots,\Lambda_K^{(i)}$ we have that for any $t \in [T],$ $\mathbb{E}(R_t)=1+ \epsilon\mathbb{P}_i(a_t=i)$. It follows that the expected reward of algorithm $A$, $G_A$ satisfies  $\mathbb{E}_i(G_A)= T+ \epsilon\mathbb{E}_i(N_i)$. The expectation in the regret measure is taken with respect to $\mathbb{P}_*$, rather than any $\mathbb{P}_i$, as such $\mathbb{E}_*(G_A)$ is the quantity of interest. We recall that under $\mathbb{P}_*$ the ``good" arm is chosen uniformly at random, and thus, it follows that \begin{align}
\mathbb{E}_*(G_A) = \frac{1}{K}\sum_{k=1}^K \mathbb{E}_k(G_a) &\leq T+ \frac{1}{K}\sum_{k=1}^K \epsilon\mathbb{E}_k(N_k) \nonumber \\
&\leq T + \frac{\epsilon}{K}\sum_{k=1}^K \bigg(\mathbb{E}_{equ}(N_k) + \frac{T}{2}\sqrt{2\epsilon^2 \left(\mathbb{E}_{equ}(N_k)\frac{\gamma_{k-1}}{2(\gamma_{k-1}-\gamma_k)} + F_k\right) }\bigg) \nonumber  \\
&= T + \frac{\epsilon T}{K} + \frac{\epsilon T}{2K}\sum_{k=1}^K \sqrt{2\epsilon^2 \left(\mathbb{E}_{equ}(N_k)\frac{\gamma_{k-1}}{2(\gamma_{k-1}-\gamma_k)} + F_k\right) }, \label{eq::expandedGA}
\end{align} where the second inequality uses Lemma \ref{lem::playsbound}.

Considering the final term of \eqref{eq::expandedGA}, we have by Cauchy-Schwarz,  \begin{align*}
\sum_{k=1}^K \sqrt{2\epsilon^2 \left(\mathbb{E}_{equ}(N_k)\frac{\gamma_{k-1}}{2(\gamma_{k-1}-\gamma_k)} + F_k\right) } &\leq \sqrt{K\sum_{k=1}^K 2\epsilon^2 \left(\mathbb{E}_{equ}(N_k)\frac{\gamma_{k-1}}{2(\gamma_{k-1}-\gamma_k)} + F_k\right)} \\
&\leq \sqrt{\epsilon^2 KT + 2\epsilon^2K\sum_{k=1}^K C(\gamma_{k-1},\gamma_k,\gamma_{k+1})\sum_{j=k}^K \gamma_j\mathbb{E}_{equ}(N_j)} \\
&\leq \sqrt{3\epsilon^2K T \max_{k \in [K]} C(\gamma_{k-1},\gamma_k,\gamma_{k+1})}
\end{align*} 
Thus \begin{align*}
\mathbb{E}_*(G_A) \leq T + \frac{\epsilon T}{K} + \frac{\epsilon T}{2} \sqrt{\frac{3\epsilon^2 T \max_{k \in [K]}C(\gamma_{k-1},\gamma_k,\gamma_{k+1})}{K} },
\end{align*} and the regret is bounded as \begin{align*}
\mathbb{E}_*|G_{max}-G_A| &\geq (1+\epsilon)T - T - \frac{\epsilon T}{K} - \frac{\epsilon T}{2} \sqrt{\frac{3\epsilon^2 T \max_{k \in [K]}C(\gamma_{k-1},\gamma_k,\gamma_{k+1})}{K}} \\
&= \epsilon T - \frac{\epsilon T}{K} - \frac{\epsilon T}{2} \sqrt{\frac{3\epsilon^2 T \max_{k \in [K]}C(\gamma_{k-1},\gamma_k,\gamma_{k+1})}{K}}. 
\end{align*}

\end{proof}

\subsection{Proof of Lemma \ref{lem::playsbound}}
We first introduce some further notation used in the proof. Define for any distributions $\mathbb{P}$ and $\mathbb{Q}$ over vector sequences $\tilde{\mathbf{r}} \in \mathbb{N}^{K\times T}$, the variational distance as $$ ||\mathbb{P}-\mathbb{Q}||_1 \equiv \sum_{\mathbf{r} \in \mathbb{N}^{K \times T}} |\mathbb{P}(\tilde{\mathbf{r}})-\mathbb{Q}(\tilde{\mathbf{r}})|,$$ and the KL divergence as $$ KL(\mathbb{P}~||~\mathbb{Q}) \equiv \sum_{\mathbf{r} \in \mathbb{N}^{K \times T}} \mathbb{P}(\mathbf{r}) \log \bigg(\frac{\mathbb{P}(\mathbf{r})}{\mathbb{Q}(\mathbf{r})} \bigg).$$ By Pinsker's inequality, we have the following relationship between these distances \begin{equation}
||\mathbb{P}-\mathbb{Q}||_1 \leq \sqrt{2KL(\mathbb{Q}~||~\mathbb{P})}. \label{eq::vartoKL}
\end{equation} Finally, the KL divergence between two Poisson distributions with parameters $\lambda$ and $\nu$ is given as, $$ KL(\lambda || \nu) \equiv \lambda \log \bigg( \frac{\lambda}{\nu}\bigg) + \nu - \lambda.$$ 

\begin{proof}
For any function $f: \mathbb{N}^{K \times T} \rightarrow [0,M]$, with $M>0$ constant, we have, \begin{align}
\mathbb{E}_i(f(\tilde{\mathbf{r}}))- \mathbb{E}_{equ}(f(\tilde{\mathbf{r}})) &= \sum_{\tilde{\mathbf{r}} \in \mathbb{N}_+^{K\times T}} f(\tilde{\mathbf{r}})\big(\mathbb{P}_i(\tilde{\mathbf{r}})-\mathbb{P}_{equ}(\tilde{\mathbf{r}})\big) \nonumber \\
&\leq \sum_{\tilde{\mathbf{r}} : \mathbb{P}_i(\tilde{\mathbf{r}}) \geq \mathbb{P}_{equ}(\tilde{\mathbf{r}})} f(\tilde{\mathbf{r}})\big(\mathbb{P}_i(\tilde{\mathbf{r}})-\mathbb{P}_{equ}(\tilde{\mathbf{r}})\big) \nonumber  \\
&\leq \frac{M}{2}||\mathbb{P}_i-\mathbb{P}_{equ}||_1 \nonumber \\
&\leq \frac{M}{2}\sqrt{2KL(\mathbb{P}_{equ}||\mathbb{P}_i)}, \label{eq::functionbound}
\end{align} where the final inequality follows from \eqref{eq::vartoKL}.  Considering the KL divergence term in isolation, we have, by Theorem 2.5.3 of \cite{cover2012elements} \begin{align*}
KL(\mathbb{P}_{equ} ~||~ \mathbb{P}_i) &= \sum_{t=1}^T KL\big(\mathbb{P}_{equ}(\tilde{\mathbf{r}}_t ~|~ \tilde{\mathbf{r}}_{1:t-1}) ~ \big|\big| ~ \mathbb{P}_i(\tilde{\mathbf{r}}_t ~|~ \tilde{\mathbf{r}}_{1:t-1}) \big) \\
&= \sum_{t=1}^T \sum_{k=1}^K \mathbb{P}_{equ}\big(a_t=k\big) KL\big( \mathbb{P}_{equ}(\tilde{\mathbf{r}}_t ~|~ a_t=k) ~ \big| \big| ~ \mathbb{P}_i(\tilde{\mathbf{r}}_t ~|~ a_t=k)\big) \\
&= \sum_{t=1}^T \sum_{k=i}^K \mathbb{P}_{equ}\big(a_t=k\big) KL\big( \mathbb{P}_{equ}(\tilde{\mathbf{r}}_t ~|~ a_t=k) ~ \big| \big| ~ \mathbb{P}_i(\tilde{\mathbf{r}}_t ~|~ a_t=k)\big) \\
&= \sum_{t=1}^T \sum_{k=i}^K \mathbb{P}_{equ}\big(a_t=k\big) \sum_{j=1}^k KL\Big(\gamma_k(\Lambda_j^{equ}-\Lambda_{j-1}^{equ}),\gamma_k(\Lambda_j^{(i)}-\Lambda_{j-1}^{(i)})\Big) \\
&=\sum_{t=1}^T \sum_{k=i}^K \mathbb{P}_{equ}\big(a_t=k\big) \sum_{j=1}^k KL\Big(\gamma_k(\frac{1}{\gamma_j}-\frac{1}{\gamma_{j-1}}),\gamma_k(\Lambda_j^{(i)}-\Lambda_{j-1}^{(i)})\Big). 
\end{align*} Here the parameters $\Lambda_{k}^{equ}$, $k \in [K]$ refer to the choice of CIF parameters which yields $\mu_k=1$ for all $k \in [K]$. The final equality follows from the observation that if $a_t < i$ then the distribution of the filtered observations is identical under $\mathbb{P}_{equ}$ and $\mathbb{P}_i$. Decomposing on the sum over $k$, with the observation that for $j > i+1$ the CIF parameters under the ``single good arm" and ``all arms equal" constructions will also match, meaning  $KL(\gamma_k(\Lambda_j^{equ}-\Lambda_{j-1}^{equ}),\gamma_k(\Lambda_k^{(i)}-\Lambda_{j-1}^{(i)}))=0,$ for any  $j > i+1 $ we have
\begin{align}
&\enspace KL(\mathbb{P}_{equ} ~||~ \mathbb{P}_i) \nonumber \\
&= \sum_{t=1}^T \mathbb{P}_{equ}(a_t=i)KL\Big(\gamma_i(\frac{1}{\gamma_i}-\frac{1}{\gamma_{i-1}}),\gamma_i(\frac{1+\epsilon}{\gamma_i}- \frac{1}{\gamma_{i-1}})\Big) \nonumber \\
&\quad + \sum_{t=1}^T \sum_{k=i+1}^K \mathbb{P}_{equ}\big(a_t=k\big) \sum_{j\in \{i,i+1\}} KL\Big(\gamma_k(\frac{1}{\gamma_j}-\frac{1}{\gamma_{j-1}}),\gamma_k(\Lambda_j^{(i)}-\Lambda_{j-1}^{(i)})\Big) \nonumber \\
&=\mathbb{E}_{equ}(N_i)\bigg((1-\frac{\gamma_i}{\gamma_{i-1}})\log\Big(\frac{1- \frac{\gamma_i}{\gamma_{i-1}}}{1+\epsilon- \frac{\gamma_i}{\gamma_{i-1}}}  \Big)+ \epsilon \bigg) \nonumber\\
&\quad + \sum_{k=i+1}^K \mathbb{E}_{equ}(N_k)\Bigg[\bigg((\frac{\gamma_k}{\gamma_i}-\frac{\gamma_k}{\gamma_{i-1}})\log\Big(\frac{\frac{1}{\gamma_i}- \frac{1}{\gamma_{i-1}}}{\frac{1+\epsilon}{\gamma_i}- \frac{1}{\gamma_{i-1}}}  \Big)+ \epsilon\frac{\gamma_k}{\gamma_i} \bigg)  \nonumber \\ 
&\quad \quad \quad \quad \quad \quad \quad \quad \quad \quad + \bigg((\frac{\gamma_k}{\gamma_{i+1}}-\frac{\gamma_k}{\gamma_i})\log\Big(\frac{\frac{1}{\gamma_{i+1}}-\frac{1}{\gamma_i}}{\frac{1}{\gamma_{i+1}}-\frac{1+\epsilon}{\gamma_i}} \Big)- \epsilon\frac{\gamma_k}{\gamma_i} \bigg)\Bigg] \nonumber \\
&\leq \mathbb{E}_{equ}(N_i)\frac{\gamma_{i-1}}{2(\gamma_{i-1}-\gamma_i)}\epsilon^2 \nonumber \\
&\quad +\sum_{k=i+1}^K \mathbb{E}_{equ}(N_k) \frac{\gamma_k}{\gamma_i}\Bigg[{\frac{\gamma_{i-1}-\gamma_i}{\gamma_{i-1}}}\log\bigg(\frac{1}{1+ \frac{\gamma_{i-1}}{\gamma_{i-1}-\gamma_i}\epsilon}\bigg) + {\frac{\gamma_i-\gamma_{i+1}}{\gamma_{i+1}}}\log\bigg(\frac{1}{1- \frac{\gamma_{i+1}}{\gamma_i-\gamma_{i+1}}\epsilon}\bigg) \Bigg], \label{eq::refinedKL}
\end{align} for $\epsilon \leq \frac{\gamma_i-\gamma_{i+1}}{\gamma_{i+1}}$. The inequality uses the identity \begin{equation*}
    2ax + 2a^2\log\left(\frac{a}{a+x}\right) \leq x^2, \enspace x>0,~ a<1.
\end{equation*} 

It remains to bound the summation in \eqref{eq::refinedKL} with an $o(\epsilon^2)$ term. For general $a \in [0,1]$, $b \in [0,1]$, and $0 \leq x \leq b$, consider the function \begin{displaymath}
g(x)= a\log\bigg(\frac{1}{1+\frac{x}{a}}\bigg) + b\log\bigg(\frac{1}{1+\frac{x}{b}}\bigg).
\end{displaymath} We have its derivative \begin{displaymath}
\frac{dg(x)}{dx} = \frac{a}{a-x}+ \frac{b}{x-b},
\end{displaymath} and thus for some $C>1$ we have the following linear bound on the derivative \begin{equation}
dg/dx \leq 2Cx, \quad \forall ~ x \in \bigg[0, \frac{a+b}{2} + \sqrt{\frac{(a+b)^2}{4}-\frac{ab+(a-b)}{4C}}\bigg]. \label{eq::range}
\end{equation} Solutions to $g(x)=Cx^2$ are not available in closed-form, but since $g(0)=0$, and $dg/dx|_{x=0}=0$ we have as a minimum that $g(x) \leq Cx^2$ for $x$ as in \eqref{eq::range}. Choosing $C=\frac{ab+a-b}{(a+b)^2}$ gives $g(x) \leq Cx^2$ for $x \in [0,\frac{a+b}{2}]$.

It therefore follows that \begin{align}
   &\enspace {\frac{\gamma_{i-1}-\gamma_i}{\gamma_{i-1}}}\log\bigg(\frac{1}{1+ \frac{\gamma_{i-1}}{\gamma_{i-1}-\gamma_i}\epsilon}\bigg) + {\frac{\gamma_i-\gamma_{i+1}}{\gamma_{i+1}}}\log\bigg(\frac{1}{1- \frac{\gamma_{i+1}}{\gamma_i-\gamma_{i+1}}\epsilon}\bigg) \nonumber\\
   &\quad \quad \quad \leq \Bigg(\frac{\gamma_i(\gamma_{i-1}-2\gamma_i+\gamma_{i+1})}{\frac{\gamma_{i+1}}{\gamma_{i-1}}(\gamma_{i-1}-\gamma_i)^2 + 2(\gamma_{i-1}-\gamma_i)(\gamma_i-\gamma_{i+1})+\frac{\gamma_{i-1}}{\gamma_{i+1}}(\gamma_i-\gamma_{i+1})^2}\Bigg)\epsilon^2, \label{eq::quadraticbound}
\end{align} for all $x \in \big[0, \frac{\gamma_i}{2\gamma_{i+1}}-\frac{\gamma_i}{2\gamma_{i-1}}\big]$.

Combining \eqref{eq::refinedKL} and \eqref{eq::quadraticbound} we therefore have that the KL divergence from $\mathbb{P}_{equ}$ to $\mathbb{P}_i$ may be bounded as follows, \begin{equation}
KL(\mathbb{P}_{equ} ~ ||~ \mathbb{P}_i) \leq  \epsilon^2\left(\frac{\gamma_{i-1}}{2(\gamma_{i-1}-\gamma_i)}\mathbb{E}_{equ}(N_i)+  C(\gamma_{i-1},\gamma_i,\gamma_{i+1}) \sum_{k=i+1}^K \gamma_k\mathbb{E}_{equ}(N_k), \right) \label{eq::finalKLbound}
\end{equation} for $\epsilon \leq \frac{\gamma_i}{2\gamma_{i+1}}-\frac{\gamma_i}{2\gamma_{i-1}}$,  where 
$C(\gamma_{i-1},\gamma_i,\gamma_{i+1})$ is a known positive constant. Finally, as $N_i: \mathbb{N}^{K\times T} \rightarrow [0,T]$, we have the stated result by the combination of \eqref{eq::functionbound}, and \eqref{eq::finalKLbound}. 

\end{proof}

\section{Experiments} \label{sec::experiments}
In this section we illustrate the performance of CIF-UCB via numerical examples. We work with a linear intensity function $\lambda(x) = 20-20x$ and exponential filtering probability $\gamma(x)=\exp(-x)$, both  for $x\in [0,1]$. The plot of $\Lambda(x)\gamma(x)$ is shown in Figure \ref{fig:lambda-gamma}, with $x^*=0.586$ and $\Lambda(x^*)\gamma(x^*)=4.61$ (found numerically). In the experiment, we set the Lipschitz constant $m=20$, which equals $\max_{0\leq x\leq 1}(\Lambda(x)\gamma(x))'$ (since $\Lambda(x)\gamma(x)$ is concave), and $\lambda_{\max}=20$. 

\begin{figure}[h]
    \begin{center}
    \includegraphics[trim =0 1.5cm 0 0.75cm, width=10cm, height=10cm]{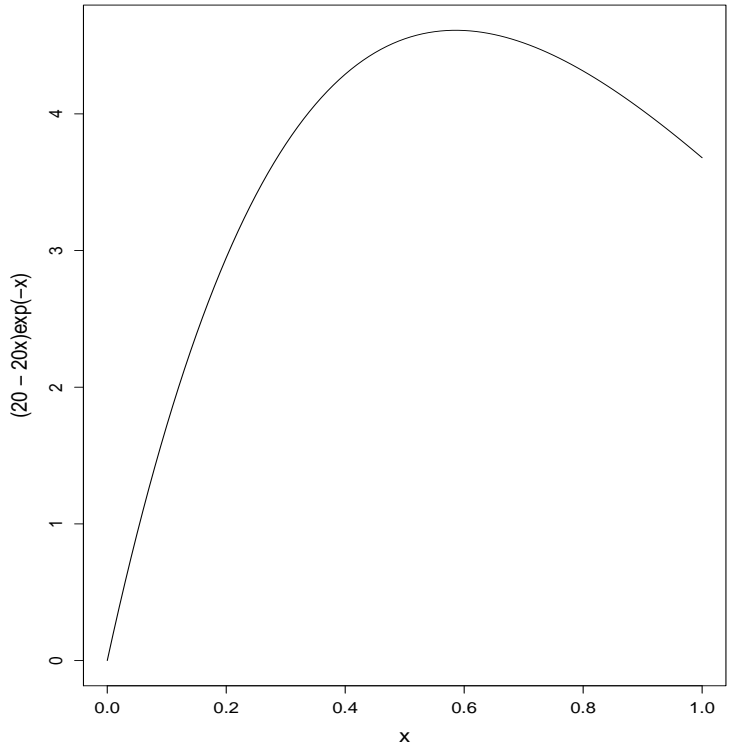}
    \end{center}
    \caption{Plot of $\Lambda(x)\gamma(x)$.}
    \label{fig:lambda-gamma}
\end{figure}

We ran 100 independent sample paths over a time horizon of $T=50000$, and computed the average cumulative regret over the 100 sample paths. The resulting average cumulative regret is shown in Figure \ref{fig:regret-base}, along with the upper regret bound, as determined in Theorem \ref{thm:upperbound}.

\begin{figure}[h]
    \begin{center}
    \includegraphics[trim =0 1.5cm 0 0.75cm, width=10cm, height=10cm]{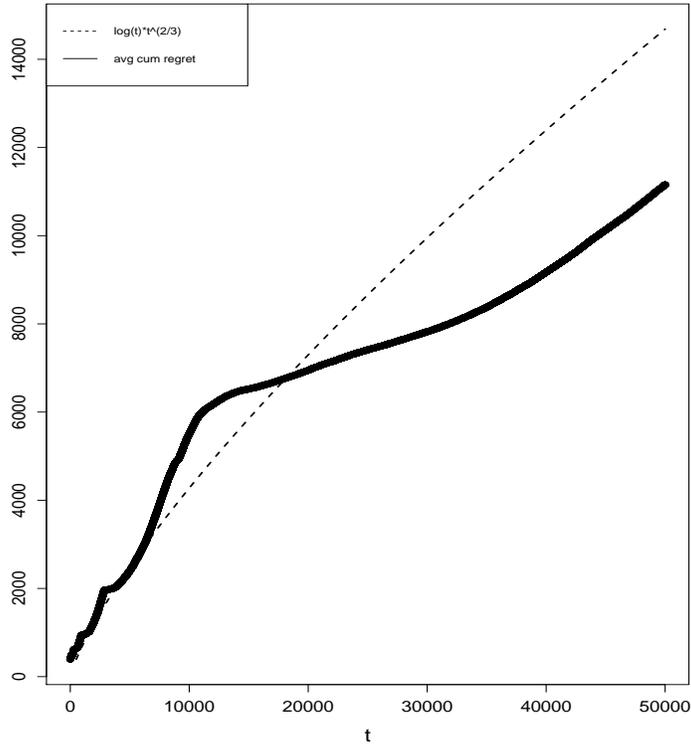}
    \end{center}
    \caption{Plot of average cumulative regret.}
    \label{fig:regret-base}
\end{figure}

Several observations are in order. First, the dotted curve in Figure \ref{fig:regret-base} doesn't include the constant terms (equal to 360 in this case) nor the sub $\log(t) t^{2/3}$ terms that come up in the regret upper bound derivation (cf. Eq. (\ref{eq:ult1})). Still, we note that the regret growth is plausibly of order  $\tilde O(t^{2/3})$. 

The second observation concerns the shape of the average cumulative regret. Note that the cumulative regret appears to be piece-wise convex increasing, such that the regret of each extra convex piece grows at a slower rate; this observation is even more noticeable on individual sample paths (not shown). This growth pattern is due to the splitting condition of CIF-UCB, whereby the algorithm initially samples the best of the two segments that result from a split, and explores other (typically worse) segments as $t$ gets larger. As $t$ grows, the algorithm exploits more often, and thus each convex piece grows slower. 

The final observation is about the splitting pattern. We include in Table \ref{table:dataframe} the data frame for the final round of a sample path in the R implementation, which includes the two endpoints ($x$ and $y$), the effective number of samples of each final segment $\sum_{i=1}^{|V_{T}(x,y)|}\gamma(b_{\tau_i})$, the index $\mathcal I_{T}(x,y)$, and the CIF estimator $\bar \Lambda_{T}(y)$ in the rightmost column. The finer spatial grid around $x^*$ is appreciable, suggesting that the algorithm gravitates towards the segment that contains the optimal solution $x^*$. Note also that the estimates of $\Lambda(x) = 20 x -10 x^2$ are very precise (the largest relative error is .62\% for $y$ large, since  the segments close to 1 have the fewest number of effective samples $\sum_{i=1}^{|V_{T}(x,y)|}\gamma(b_{\tau_i})$). The index values are similar across the final segments, as is typical with UCB algorithms, and the effective number of samples drops off significantly to the right of $x^*$. On the other hand, the effective number of samples to the left of $x^*$ is large, since the algorithm needs to cover that space to reach (and exploit) the neighborhood around $x^*$.

\begin{table}[ht]   
\tiny
\begin{center}
\begin{tabular}{|c|c|c|c|c|}
  \hline
   $x$ &   $y$ & $\sum_{i=1}^{|V_{T}(x,y)|}\gamma(b_{\tau_i})$ &       $\mathcal I_{T}(x,y)$ & $\bar \Lambda_T(y)$  \\ 
  \hline
0.0000000 &  0.1250000 &  24118.996 & 4.740275 & 2.348149 \\
0.1250000 &  0.1875000 &  24118.996 & 4.225897 & 3.398939 \\
0.1875000 &  0.2500000 &  24118.996 & 4.802038 & 4.370248 \\
0.2500000 &  0.2812500 &  24117.438 & 4.413122 & 4.827627 \\
0.2812500 &  0.3125000 &  24117.438 & 4.615670 & 5.263785 \\
0.3125000 &  0.3437500 &  24116.707 & 4.794795 & 5.689541 \\
0.3437500 &  0.3750000 &  24116.707 & 4.943961 & 6.093245 \\
0.3750000 &  0.3906250 &  24115.332 & 4.692382 & 6.282144 \\
0.3906250 &  0.4062500 &  24115.332 & 4.747511 & 6.466840 \\
0.4062500 &  0.4218750 &  24114.666 & 4.797825 & 6.648402 \\
0.4218750 &  0.4375000 &  24114.666 & 4.843334 & 6.826593 \\
0.4375000 &  0.4531250 &  24113.375 & 4.882826 & 6.999228 \\
0.4531250 &  0.4687500 &  24113.375 & 4.916710 & 7.166604 \\
0.4687500 &  0.4843750 &  24112.749 & 4.946188 & 7.330313 \\
0.4843750 &  0.4921875 &  24112.749 & 4.803775 & 7.411888 \\
0.4921875 &  0.5000000 &  24112.749 & 4.814192 & 7.488694 \\
0.5000000 &  0.5078125 &  23996.902 & 4.825221 & 7.566727 \\
0.5078125 &  0.5156250 &  23996.902 & 4.834771 & 7.643570 \\
0.5156250 &  0.5234375 &  23996.304 & 4.845463 & 7.723064 \\
0.5234375 &  0.5312500 &  23996.304 & 4.852431 & 7.796992 \\
0.5312500 &  0.5390625 &  23995.129 & 4.858238 & 7.869597 \\
0.5390625 &  0.5468750 &  23995.129 & 4.861613 & 7.938653 \\
0.5468750 &  0.5546875 &  23994.550 & 4.865409 & 8.009027 \\
0.5546875 &  0.5625000 &  23994.550 & 4.868062 & 8.078001 \\
0.5625000 &  0.5703125 &  23875.465 & 4.869991 & 8.145726 \\
0.5703125 &  0.5781250 &  23875.465 & 4.870570 & 8.212154 \\
0.5781250 &  0.5859375 &  23726.253 & 4.870049 & 8.276444 \\
0.5859375 &  0.5937500 &  23726.253 & 4.869350 & 8.341604 \\
0.5937500 &  0.6015625 &  23314.273 & 4.869802 & 8.407425 \\
0.6015625 &  0.6093750 &  23314.273 & 4.867215 & 8.470133 \\
0.6093750 &  0.6171875 &  21772.366 & 4.866196 & 8.528425 \\
0.6171875 &  0.6250000 &  21772.366 & 4.861871 & 8.588823 \\
0.6250000 &  0.6328125 &  20964.657 & 4.860349 & 8.650464 \\
0.6328125 &  0.6406250 &  20964.657 & 4.854130 & 8.708132 \\
0.6406250 &  0.6484375 &  18021.041 & 4.851132 & 8.753434 \\
0.6484375 &  0.6562500 &  18021.041 & 4.843125 & 8.808425 \\
0.6562500 &  0.6640625 &  16849.088 & 4.842293 & 8.868670 \\
0.6640625 &  0.6718750 &  16849.088 & 4.833623 & 8.923094 \\
0.6718750 &  0.6796875 &  14015.447 & 4.831297 & 8.963610 \\
0.6796875 &  0.6875000 &  14015.447 & 4.820711 & 9.014911 \\
0.6875000 &  0.6953125 &  11653.144 & 4.823237 & 9.062533 \\
0.6953125 &  0.7031250 &  11653.144 & 4.811218 & 9.111618 \\
0.7031250 &  0.7187500 &   7729.984 & 4.954881 & 9.151894 \\
0.7187500 &  0.7343750 &   6535.462 & 4.945111 & 9.240815 \\
0.7343750 &  0.7500000 &   6535.462 & 4.910583 & 9.319769 \\
0.7500000 &  0.7656250 &   6072.071 & 4.885203 & 9.399100 \\
0.7656250 &  0.7812500 &   6072.071 & 4.848842 & 9.474527 \\
0.7812500 &  0.7968750 &   5608.286 & 4.820870 & 9.546411 \\
0.7968750 &  0.8125000 &   5608.286 & 4.778192 & 9.607749 \\
0.8125000 &  0.8281250 &   4692.391 & 4.770177 & 9.693565 \\
0.8281250 &  0.8437500 &   4692.391 & 4.723798 & 9.746417 \\
0.8437500 &  0.8593750 &   3776.720 & 4.714856 & 9.810630 \\
0.8593750 &  0.8750000 &   3776.720 & 4.664373 & 9.853260 \\
0.8750000 &  0.9062500 &   2290.606 & 4.954441 & 9.899998 \\
0.9062500 &  0.9375000 &   1757.278 & 4.886901 & 9.906228 \\
0.9375000 &  0.9687500 &   1236.443 & 4.740351 & 9.937378 \\
0.9687500 &  1.0000000 &   1236.443 & 4.746599 & 9.954363 \\

 \hline
\end{tabular}
\caption{Summary of main parameters after a sample path}
\end{center}
 \label{table:dataframe}
\end{table}

To test the sensitivity of the algorithm to multiple local maximums, we ran a second experiment with parameters identical to those of the first experiment, except for the filtering probability $\gamma(\cdot)$, which now is set to be piece-wise linearly decreasing, 
\[
\gamma(x) = \begin{cases}
  1, \text{ for }x \in [0, 0.25)\\      
  1.5 - 2x  \text{ for }x \in [0.25, 0.5) \\
  0.5, \text{ for }x \in [0.5, 0.8)\\   
  1.3 - x  \text{ for }x \in [0.8, 1]. 
\end{cases}
\]
This filtering probability leads to a $\Lambda(x)\gamma(x)$ objective as in Figure \ref{fig:lambda-gamma2}, with $x^*=0.8$ and $\Lambda(x^*)\gamma(x^*)=4.8$.

\begin{figure}[h]
    \begin{center}
    \includegraphics[trim =0 1.5cm 0 0.75cm, width=10cm, height=10cm]{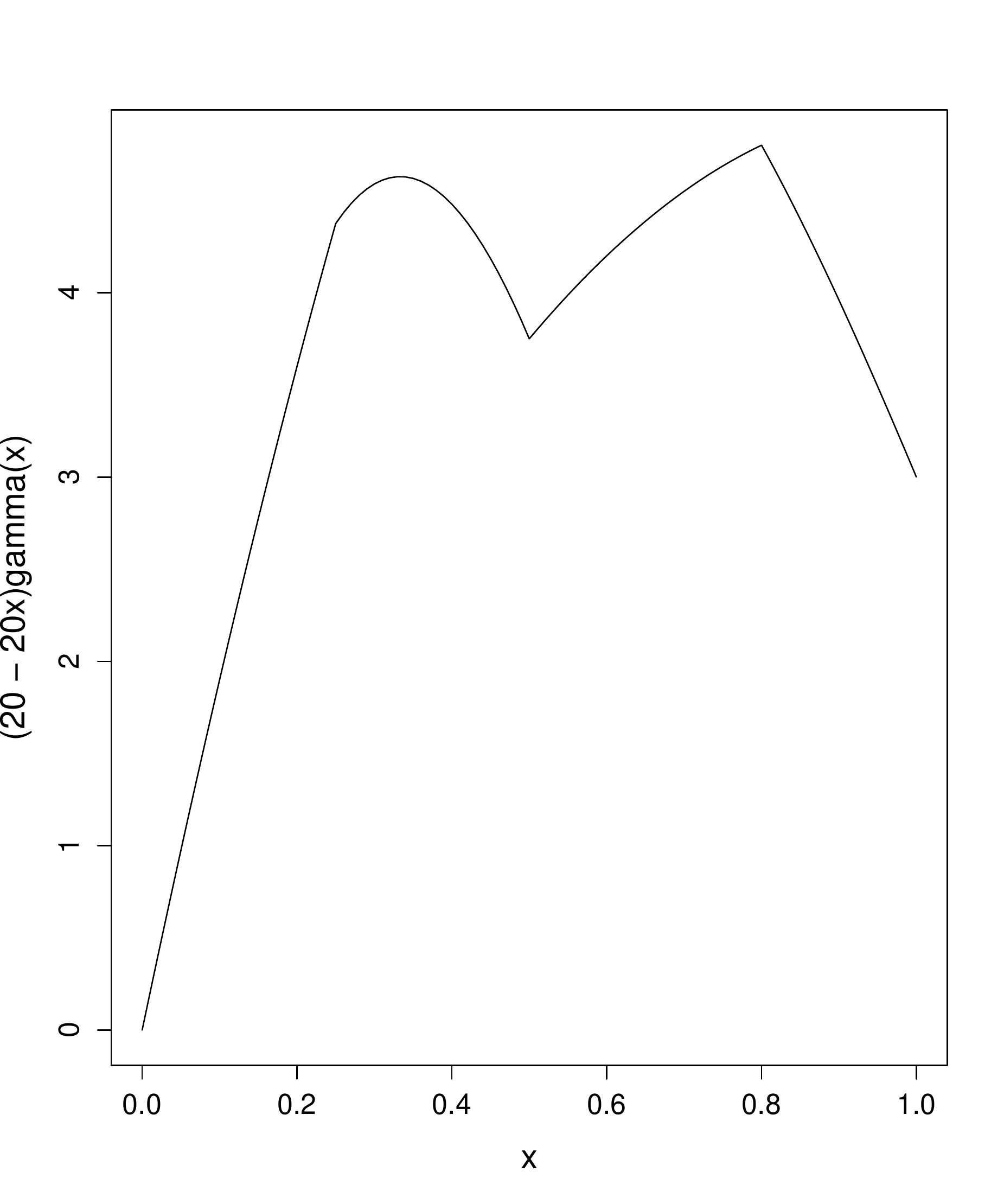}
    \end{center}
    \caption{Plot of $\Lambda(x)\gamma(x)$.}
    \label{fig:lambda-gamma2}
\end{figure}

We tested CIF-UCB over 100 independent sample paths, with a time horizon $T=50000$. This resulted in an average cumulative regret as shown in Figure \ref{fig:regret-base2}. 

\begin{figure}[h]
    \begin{center}
    \includegraphics[trim =0 1.5cm 0 0.75cm, width=10cm, height=10cm]{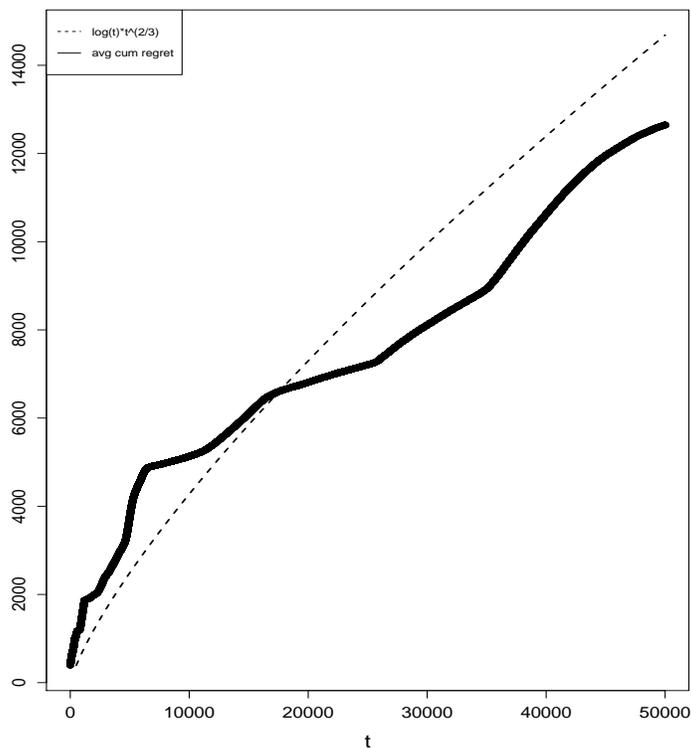}
    \end{center}
    \caption{Plot of average cumulative regret.}
    \label{fig:regret-base2}
\end{figure}

Two main observations can be drawn. First, the $\tilde O(T^{2/3})$ upper bound of Theorem \ref{thm:upperbound} holds over $t \in \{1,\ldots, T\}$. Second, the average cumulative regret is about 10\% larger than in the first experiment for $t=T$. This can be ascribed to the fact that the optimal value of the objective function is 4.8 versus 4.61 in the first experiment, and to the extra exploration induced by the local maximum at $x=.33$. 

\section{Discussion} \label{sec::discussion}

This work considers a sequential variant of the problem faced by a decision-maker who attempts to maximise the detection of events generated by a filtered non-homogeneous Poisson process, where the filtering probability depends on the segment selected by the decision-maker, and the Poisson cumulative intensity function is unknown. The independent increment property of the Poisson process makes the analysis tractable, enabling the use of the machinery developed for the continuum bandit problem. 
The problem of efficient exploration/exploitation of a filtered Poisson process on a continuum arises naturally in settings where observations are made by searchers (representing cameras, sensors, robotic and human searchers, etc.), and the events that generate observations tend to disappear (or renege, in a queueing context), before an observation can be made, as the interval of search increases. Besides extending the state-of-the-art to such settings, the main contributions are an algorithm for a filtered Poisson process on a continuum, and regret bounds that are optimal up to a logarithmic factor.

\noindent \textbf{Acknowledgements} JAG was supported by EPSRC grant EP/L015692/1 (STOR-i Centre for Doctoral Training). RS was supported by ONR grant N0001420WX00860.

\allowdisplaybreaks

\appendix

\bibliography{main}
\bibliographystyle{apalike}

\end{document}